\documentclass{article}
\usepackage{spconf,amsmath,amsthm,amssymb}
\usepackage{microtype}
\usepackage{graphicx,epsfig,url}
\usepackage[linesnumbered,boxed]{algorithm2e}
\usepackage[caption=false,font=footnotesize]{subfig}
\usepackage[usenames, dvipsnames]{color}

\title{Fast Bilateral Filtering of Vector-Valued Images}
\name{Sanjay Ghosh and Kunal N. Chaudhury
\thanks{Correspondence: kunal@ee.iisc.ernet.in. This work was supported by a Startup Grant from the Indian Institute of Science.}
}
\address{Department of Electrical Engineering, Indian Institute of Science.}

\newtheorem{theorem}{Theorem}[section]
\newtheorem{proposition}[theorem]{Proposition}

\def\X{\mathbf{X}}
\def\Y{\mathbf{Y}}
\def\x{\boldsymbol{x}}
\def\y{\boldsymbol{y}}
\def\f{\boldsymbol{f}}
\def\g{\boldsymbol{g}}

\def\i{\boldsymbol{i}}
\def\j{\boldsymbol{j}}

\begin{document}

\ninept

\maketitle

\begin{abstract}
In this paper, we consider a natural extension of the edge-preserving bilateral filter for vector-valued images. 
The direct computation of this non-linear filter is slow in practice. 
We demonstrate how a fast algorithm can be obtained by first approximating the Gaussian kernel of the bilateral filter using raised-cosines, and then using Monte Carlo sampling.
We present simulation results on color images to demonstrate the accuracy of the algorithm and the speedup over the direct implementation. 
\end{abstract}


\section{Introduction}

The bilateral filter was proposed by Tomasi and Maduchi \cite{Tomasi1998} as a non-linear extension of the classical Gaussian  filter. It is an instance of an edge-preserving filter that can smooth homogenous regions, while preserving sharp edges at the same time. The bilateral filter has diverse applications in image processing, computer vision, computer graphics, and computational photography \cite{Book2009}. We refer the interested reader to \cite{Tomasi1998,Book2009} for a comprehensive survey of the working of the filter and its various applications.

The bilateral filter uses a spatial kernel along with a range kernel to perform edge-preserving smoothing. Before proceeding further, we introduce the necessary notation and terminology.
Let $\f : \Omega \rightarrow \mathbb{R}^d$ be a vector-valued image, where $\Omega$ is some finite rectangular domain of $\mathbb{Z}^2$. For example, $d=3$ for a color image.
Consider the kernels  $\omega : \mathbb{Z}^2 \rightarrow \mathbb{R}$ and $\phi : \mathbb{R}^d \rightarrow \mathbb{R}$ given by
\begin{equation}
\label{kernel}
\omega (\i) = \exp \left(- \frac{|| \i ||^2}{2 \sigma_s^2} \right) \ \text{ and } \ \phi(\x) = \exp \Big(- \frac{1}{2} \x^T \mathbf{C}^{-1} \x \Big),
\end{equation}
where $\sigma_s > 0$ and $\mathbf{C}$ is some positive definite covariance matrix.
The former bivariate Gaussian is called the spatial kernel,  and the latter multivariate Gaussian is called the range kernel \cite{Tomasi1998}. The output of the bilateral filter is the 
vector-valued image $\mathcal{B}[\f]: \Omega \rightarrow \mathbb{R}^d$ given by
\begin{equation}
 \label{BF}
 \mathcal{B}[\f](\i) = \frac{1}{Z(\i)} \sum_{\j} \omega (\j) \  \phi\! \left( \f(\i-\j) - \f(\i) \right) \f(\i-\j),
\end{equation}
where 
\begin{equation}
\label{normal}
Z(\i)=\sum_{\j} \omega (\j)\ \phi\! \left( \f(\i-\j) - \f(\i) \right).
\end{equation}
The bilateral filter was originally proposed for processing grayscale and color images corresponding to $d=1$ and $d=3$ respectively \cite{Tomasi1998}.

In practice, the sums in \eqref{BF} and  \eqref{normal} are restricted to a square window $[-3\sigma_s,3\sigma_s]^2$ around the pixel of interest, where $\sigma_s$ is the standard deviation of the spatial Gaussian \cite{Tomasi1998}. Thus, the direct computation of \eqref{BF} and  \eqref{normal} requires $O(\sigma_s^2)$ operations per pixel. In fact, the direct implementation is known to be slow for practical settings of $\sigma_s$ \cite{Book2009,Paris2006}. For the case $d=1$ (grayscale images), researchers have come up with several  fast algorithms based on various forms of approximations \cite{Paris2006,Weiss2006,Porikli2008,Yang2009,Chaudhury2011,Kamata2015,Chaudhury2016}. A detailed account of some of the recent fast algorithms, and a comparison of their performances, can be found in \cite{Kamata2015}. A straightforward way of extending the above fast algorithms to vector-valued images is to apply the algorithm separately on each of the $d$ components. The output in this case will generally be different from that obtained using the formulation in \eqref{BF}. In this regard, it was observed in \cite{Tomasi1998,Paris2006} that the component-wise filtering of RGB images can often lead to color  distortions. It was shown that such distortions can be avoided by applying \eqref{BF} in the CIE-Lab color space, where the covariance $\mathbf{C}$ is chosen to be diagonal.

In this paper, we present a fast algorithm for computing \eqref{BF}. The core idea is that of using raised-cosines to approximate the range kernel $\phi(\x)$. This approximation  was originally proposed in \cite{Chaudhury2011} for deriving a fast algorithm for gray-scale images. It was later shown in \cite{Chaudhury2011a,Chaudhury2013} that the raised-cosine approximation can be extended for performing high-dimensional filtering using the product of one-dimensional approximations. Unfortunately, this did not lead to a practical fast algorithm. The fundamental difficulty in this regard is the so-called ``curse of dimensionality''. Namely, while a raised-cosine of small order, say $10$, suffices to approximate a one-dimensional Gaussian, the product of such approximations result in an order of $10^d$ in dimensions $d>1$. A similar bottleneck arises in the context of computing \eqref{BF} using the raised-cosine approximation. Nevertheless, we will demonstrate how this problem can be circumvented using Monte Carlo approximation \cite{DW2005}. 

The contribution and organization of the paper are as follows. In Section \ref{PA}, we extend the shiftable approximation in \cite{Chaudhury2011} for the bilateral filtering of vector-valued images given by \eqref{BF}. In this direction, we propose a stochastic interpretation of the raised-cosine approximation, and show how it can be made practical using Monte Carlo sampling. Based on this approximation, we develop a fast algorithm  in Section \ref{FA}. As an application, we use the proposed algorithm for filtering color images in Section \ref{results}. The results reported in this section demonstrate the accuracy of the approximation, and the speedup achieved over the direct implementation. We conclude the paper in Section \ref{conclusion}.


 \begin{figure}
 \centering
  \includegraphics[width=0.45\textwidth]{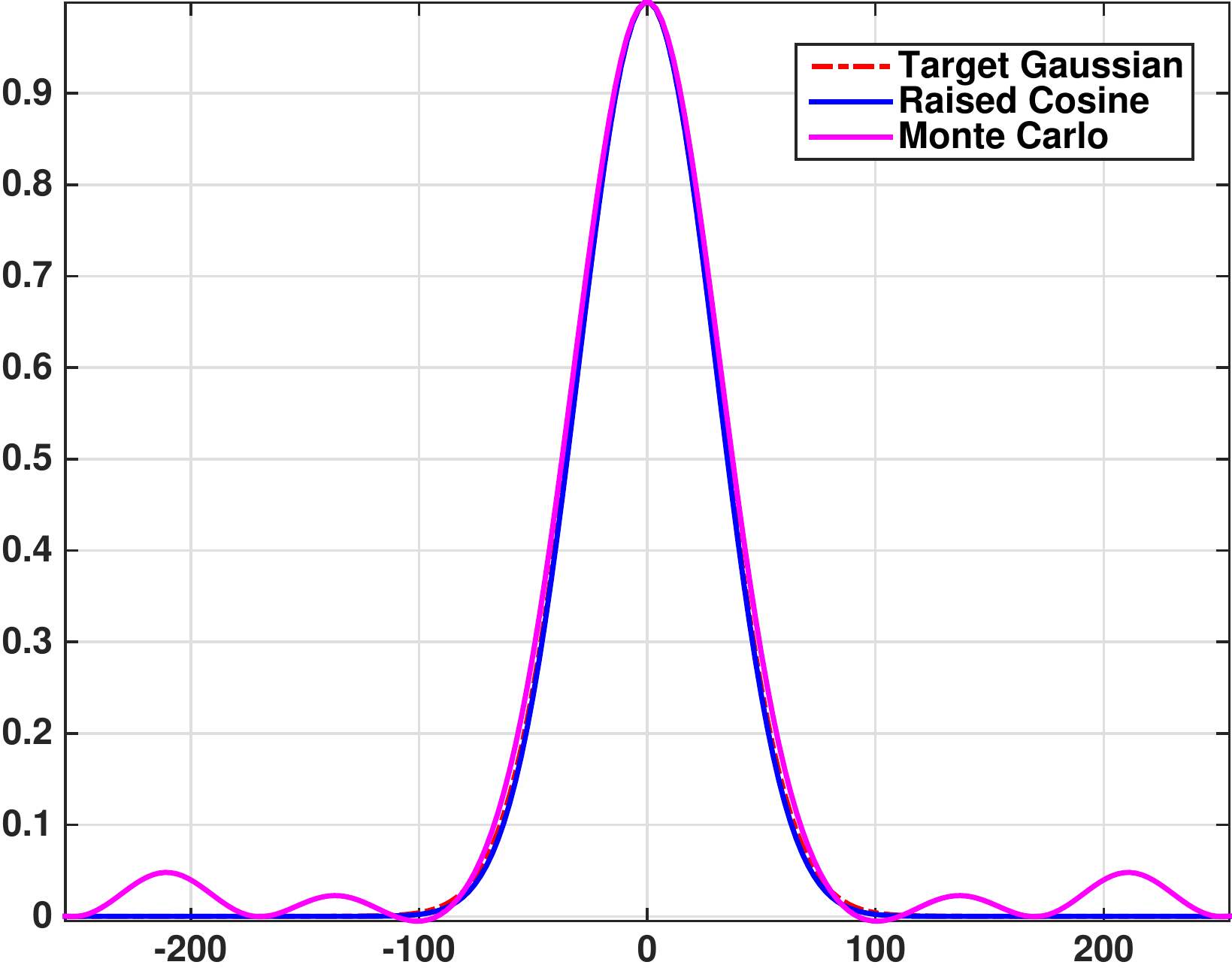}
  \caption{Approximation of the range kernel $\psi(t)=\exp (- \alpha^2  t^2/2)$ over the dynamic range $[-255,255]$, where $\alpha=1/30$. The raised cosine is of order $N=20$, and $T=200$ for the Monte Carlo approximation (obtained using a single realization of $X_1$).} 
  \label{kernelApproximation} 
\end{figure}

\section{Proposed Approximation} 
\label{PA}

As a first step, we diagonalize the quadratic form in the definition of $\phi(\x)$ using an orthogonal transform. Indeed, since $\mathbf{C}^{-1}$ is positive definite, we can find an orthogonal matrix $\mathbf{Q} \in \mathbb{R}^{d \times d}$ and positive numbers $\alpha_1,\ldots,\alpha_d$ such that
\begin{equation}
\label{COB}
 \x^T \mathbf{C}^{-1} \x = \sum_{k=1}^d \alpha^2_k  y_k^2 \qquad \quad (\y = \mathbf{Q}^T\x),
\end{equation}
where $\y=(y_1,\ldots,y_d)$. The change-of-basis $\y = \mathbf{Q}^T\x$ amounts to transforming the components of the vector-valued image at each pixel (similar to a color transformation). In particular, by defining the image $\g : \Omega \rightarrow \mathbb{R}^d$ given by
\begin{equation*}
\g(\i) = \mathbf{Q}^T\! \f(\i),
\end{equation*}
we can write \eqref{BF} as
\begin{equation}
 \label{redifBF}
 \mathcal{B}[\f](\i) = \frac{1}{Z(\i)} \sum_{\j} \omega (\j)  \ \psi \!\left( \g(\i-\j) - \g(\i) \right) \f(\i-\j),
\end{equation}
where 
\begin{equation}
\label{redifnormal}
Z(\i)=\sum_{\j} \omega (\j) \ \psi \! \left( \g(\i-\j) - \g(\i) \right).
\end{equation}
Comparing \eqref{COB} and the range kernel in \eqref{kernel}, we see that  $\psi : \mathbb{R}^d \rightarrow \mathbb{R}$ is given by
\begin{equation}
\label{modkernel}
\psi(\y) = \exp\! \left(- \frac{1}{2} \sum_{k=1}^d \alpha^2_k  y_k^2 \right) = \prod_{k=1}^d \exp\! \left(- \frac{1}{2} \alpha^2_k  y_k^2 \right).
\end{equation}
The key point here is that we have factored the original kernel $\phi(\x)$ into a product of one-dimensional Gaussians. At this point, we recall the following result from \cite{Chaudhury2011}.
\begin{proposition}[Raised-Cosine Approximation]
\label{prop1}
\begin{equation}
\label{RCapprox}
\lim_{N \rightarrow \infty} \ \left[\cos\!\left (\frac{\alpha t}{\sqrt{N}}\right)\right]^N\! = \exp \!\left(- \frac{1}{2} \alpha^2  t^2 \right).
\end{equation}
\end{proposition}
We apply Proposition \ref{prop1} to the one-dimensional Gaussians in \eqref{modkernel}, and conclude that 
\begin{equation*}
\psi(\y) =   \prod_{k=1}^d  \lim_{N \rightarrow \infty}  \left[\cos\!\left (\frac{\alpha_k y_k}{\sqrt{N}}\right)\right]^N \! \! \!=\lim_{N \rightarrow \infty}  \prod_{k=1}^d \left[\cos\!\left (\frac{\alpha_k y_k}{\sqrt{N}}\right)\right]^N\!\!.
\end{equation*}
In practice, we fix some $N$, and replace \eqref{modkernel} with 
\begin{equation}
\label{approx}
\psi(\y) =  \prod_{k=1}^d \left[\cos\!\left (\frac{\alpha_k y_k}{\sqrt{N}}\right)\right]^N.
\end{equation}
To reduce unnecessary symbols, we have used $\psi(\y)$ to represent both the original kernel and its approximation. This should not lead to a confusion, since we will not use the original kernel in the rest of the discussion. We will refer to $N$ as the \textit{order} of the approximation.

The central observation of the paper is the following stochastic interpretation of \eqref{approx}, which is obtained by turning a product-of-sums into a sum-of-products.  
\begin{proposition} 
\label{prop2}
Let $\X=(X_1,\ldots,X_d)$ be a random vector, whose components are independent and follow the binomial distribution $B(N,1/2)$.  Then
\begin{equation}
\label{MCinterp}
\psi(\y) = \mathbb{E}  \left\{  \prod_{k=1}^d \exp \!\big(\! \iota (N-2X_k) \gamma_k y_k \big)  \right\},
\end{equation}
where $\iota=\sqrt{-1}, \gamma_k=\alpha_k/ \sqrt{N},$ and the expectation $\mathbb{E}[\cdot]$ is with respect to $\X$.
\end{proposition}

\begin{proof}
Recall the identity $\cos \theta= (e^{\iota \theta} + e^{-\iota \theta})/2$, where $\iota =\sqrt{-1}$. We use this along with the binomial theorem, and get
\begin{align}
\label{expansion}
\left(\cos \theta \right)^N & = \frac{1}{2^N}\left(\exp(\iota \theta) +\exp(-\iota \theta) \right)^N \nonumber \\ 
& = \sum_{n=0}^N \frac{1}{2^N} \binom{N}{n} \exp \!\left( \iota (N-2n) \theta \right). 
\end{align}
Notice that the coefficient in \eqref{expansion} corresponding to a given $n$ is simply the probability that a random variable $X \sim B(N,1/2)$ takes on the value $n$. In other words, we can write
\begin{equation}
\label{exp}
\left(\cos \theta \right)^N = \mathbb{E} \Big\{ \exp \!\left( \iota (N-2X) \theta \right) \Big\},
\end{equation}
where $\mathbb{E}[\cdot]$ denotes the mathematical expectation with respect to $X$. Substituting \eqref{exp} in \eqref{approx}, and using the independence assumption on the components of $\X$, we arrive at \eqref{MCinterp}.
\end{proof}

\begin{figure*}
\centering
\subfloat[\textit{Dome} ($876 \times 584$).]{\includegraphics[width=0.33\linewidth]{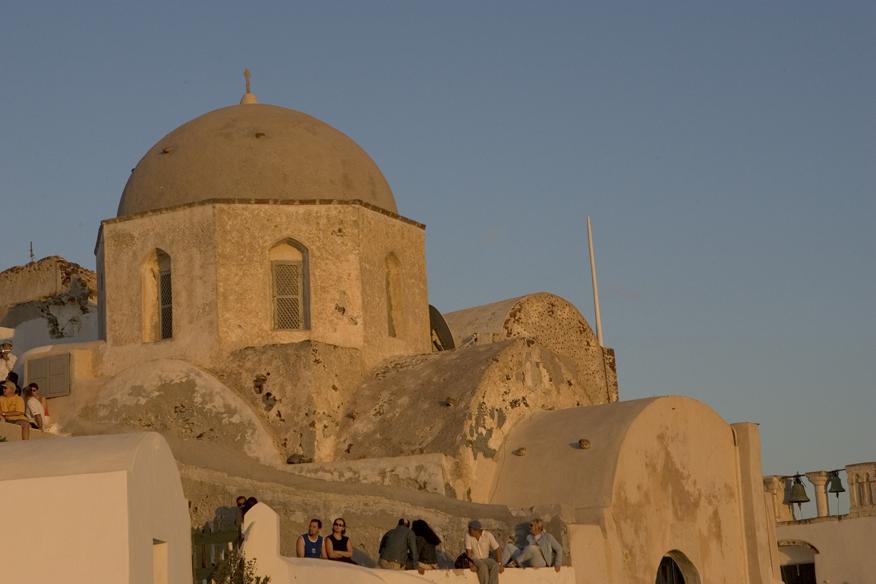}} \hspace{0.1mm} 
\subfloat[Direct, $75$ min.]{\includegraphics[width=0.33\linewidth]{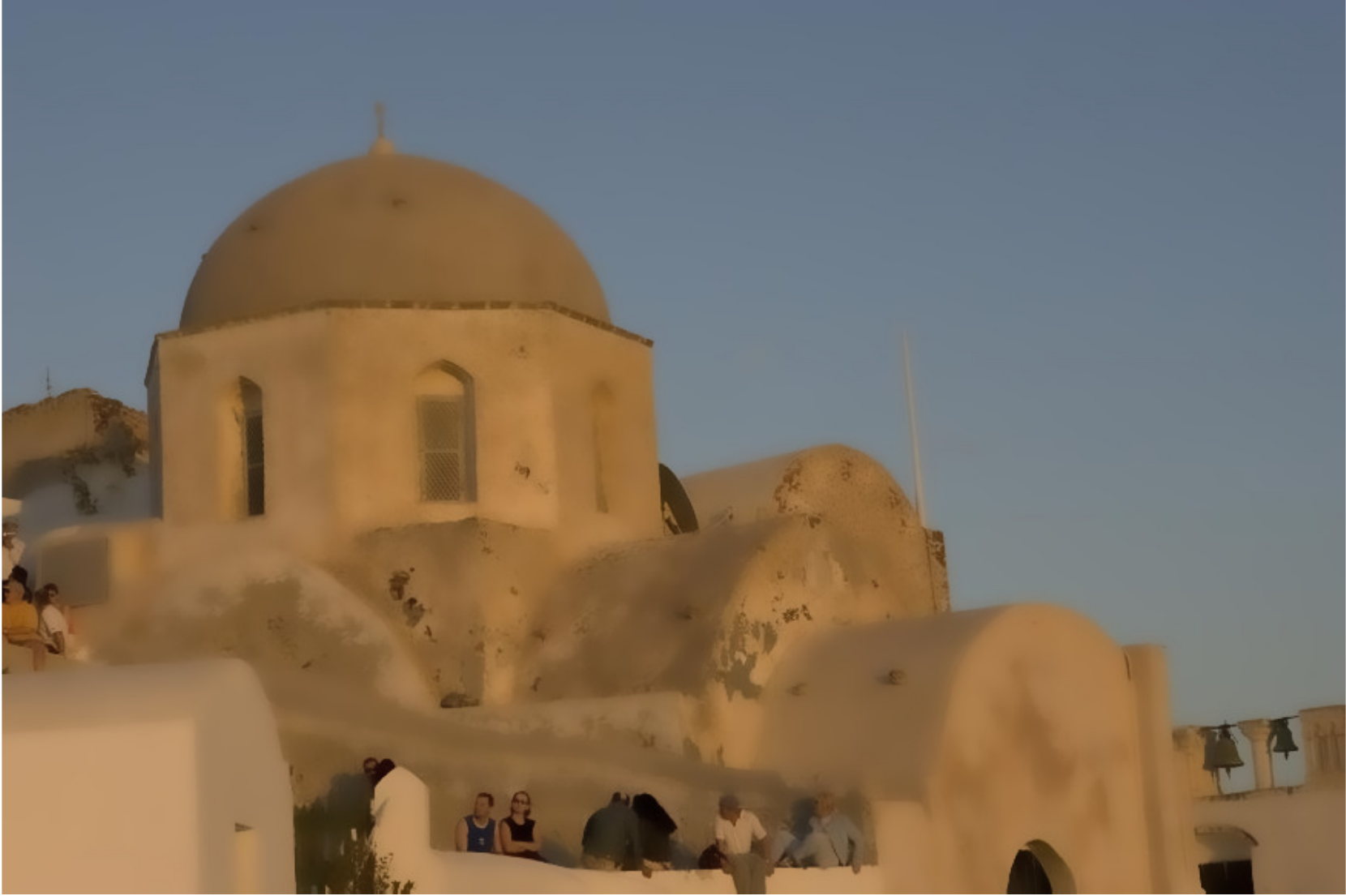}} 
\hspace{0.1mm} 
\subfloat[\texttt{MCSF}, $24$ sec.]{\includegraphics[width=0.33\linewidth]{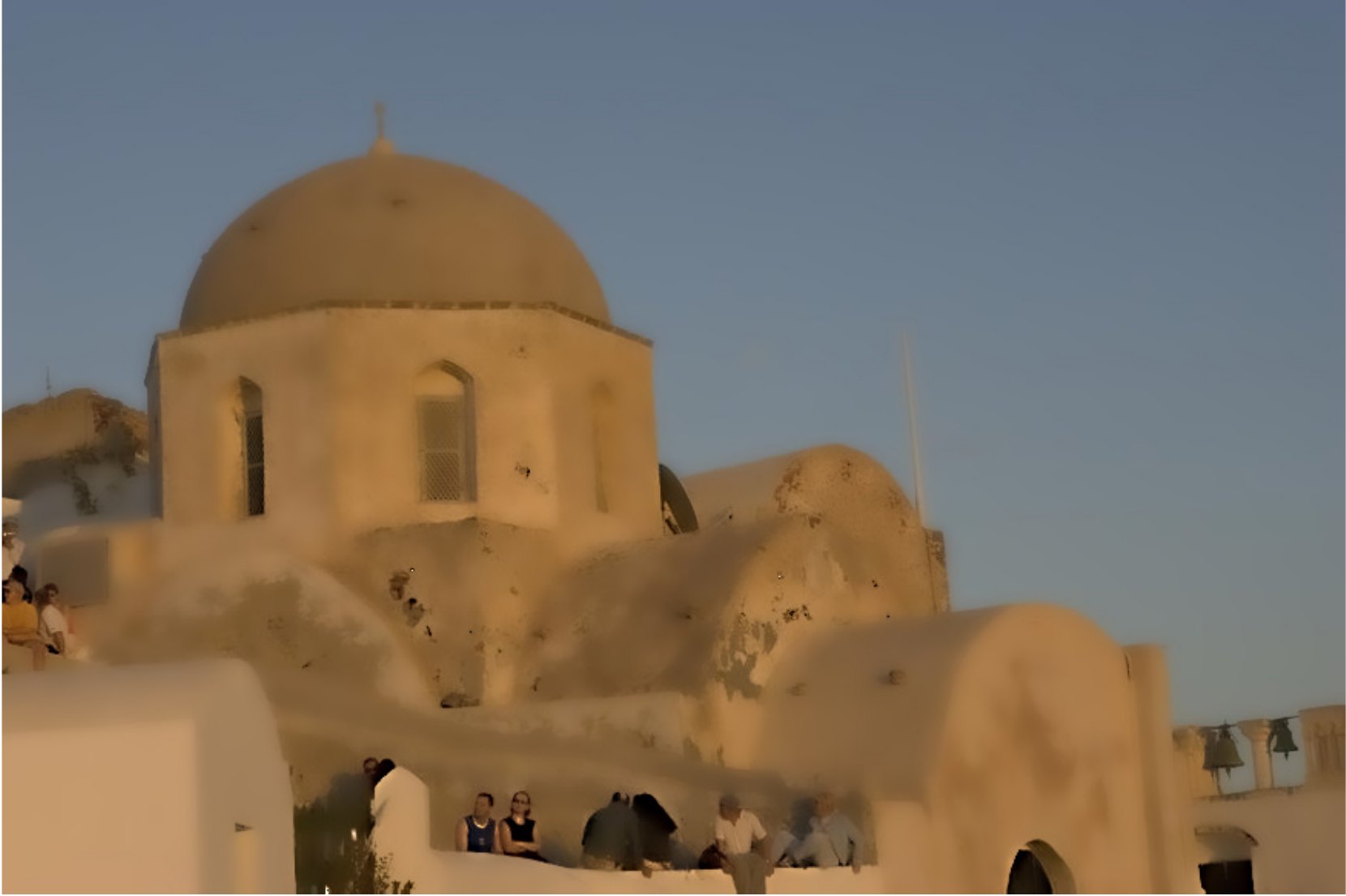}} 
\caption{Results  for the \emph{Dome} image \cite{Paris2006} at $\sigma_s=5$ and $\sigma_r = 50$. The parameters for \texttt{MCSF} are $N = 10$ and $T=200$. The MSE is  $1.86$ dB.}
\label{Dome}
\end{figure*}

We next approximate the mathematical expectation in \eqref{MCinterp} using Monte Carlo integration \cite{DW2005}. More specifically, let $F_{\X}$ denote the distribution of the random vector $\mathbf{X}$ in Proposition \ref{prop2} that takes values in $\{0,1,\ldots,N\}^d$. We fix a certain number of trials $T$, and draw $\X^{(1)},\ldots,\X^{(T)}$ from $\{0,1,\ldots,N\}^d$ using the distribution $F_{\X}$. This gives us the following Monte Carlo approximation of \eqref{MCinterp}:
\begin{equation}
\label{finalApprox}
 \psi(\y)=\frac{1}{T} \sum_{t=1}^T \left\{  \prod_{k=1}^d \exp \!\big(\! \iota (N-2X^{(t)}_k) \gamma_k y_k \big)  \right\},
\end{equation}
where $X^{(t)}_k$ is the $k$-th component of $\X^{(t)}$. To reduce symbols, we have again overloaded the notation for the original kernel in \eqref{modkernel}.

In summary, starting from the kernel in \eqref{modkernel}, we arrive at the approximation in \eqref{finalApprox} in two steps. First, we approximate the one-dimensional Gaussian using a raised-cosine of fixed order $N$. This results in a product of raised-cosines, which we express as a large sum-of-products. In the second step, we approximate the large sum using Monte Carlo integration. A comparison of the raised-cosine and the Monte Carlo approximations for a one-dimensional ($d=1$) range kernel is provided in Figure \ref{kernelApproximation}.

Before proceeding further, we simplify the expression in \eqref{finalApprox} by defining the random variable $\Y=(Y_1,\ldots,Y_d)$ given by
\begin{equation*}
Y_k = N-2X_k \qquad \ (k=1,\ldots,d).
\end{equation*}
In terms of the realizations $\Y^{(1)},\ldots,\Y^{(T)}$, we can write  \eqref{finalApprox}  as
\begin{equation}
\label{finalApproxY}
 \psi(\y)=\frac{1}{T} \sum_{t=1}^T \left\{  \prod_{k=1}^d \exp \!\big( \iota Y^{(t)}_k \gamma_k y_k \big)  \right\}.
\end{equation}

\section{Fast Algorithm} 
\label{FA}

We now present the fast algorithm obtained by using \eqref{finalApproxY} in place of the original Gaussian kernel. The fast algorithm is based on the so-called \textit{shiftability} property of a function \cite{Chaudhury2013}. The shiftability of the exponential function follows from a rather simple fact, namely that $\exp(a+b)=\exp(a) \exp(b)$. In particular, on substituting \eqref{finalApproxY} in \eqref{redifBF}, we can express the term $\psi \!\left( \g(\i-\j) - \g(\i) \right)$ as follows:
\begin{equation}
\label{bigExp}
\frac{1}{T} \sum_{t=1}^T  \Big\{ \! \prod_{k=1}^d \exp \! \big(\!- \iota Y^{(t)}_k \gamma_k  g_k(\i)\big) \exp \! \big( \iota Y^{(t)}_k \gamma_k  g_k(\i-\j)\big)  \Big\}
\end{equation}
where $g_k(\i)$ denotes the $k$-th component of $\g(\i)$. We can simplify \eqref{bigExp} by introducing the (complex-valued) image $h_{k,t} : \Omega \rightarrow \mathbb{C}$ given by
\begin{equation*}
h_{k,t}(\i) = \exp ( \iota Y^{(t)}_k \gamma_k  g_k(\i)),
\end{equation*}
and the image $H_t : \Omega \rightarrow \mathbb{C}$ given by
\begin{equation}
\label{Ht}
H_t(\i) =\prod_{k=1}^d  h_{k,t}(\i).
\end{equation}
In term of \eqref{Ht}, note that we can write \eqref{bigExp} as
\begin{equation}
\label{smallExp}
\frac{1}{T} \sum_{t=1}^T H_t(\i)^{\star} H_t(\i-\j),
\end{equation}
where $H_t(\i)^{\star}$ denotes the complex-conjugate of $H_t(\i)$. Substituting \eqref{smallExp} in \eqref{redifBF}, and exchanging the two sums, we obtain
 \begin{align}
 \label{numerator}
& \sum_{\j} \omega (\j)  \ \psi \!\left( \g(\i-\j) - \g(i) \right) \f(\i-\j) \nonumber \\
 & = \frac{1}{T}  \sum_{\j}  \omega(\j) \Big\{ \sum_{t=1}^T H_t(\i)^{\star} H_t(\i-\j) \Big\} \f(\i-\j) \nonumber \\
 & =  \frac{1}{T} \sum_{t=1}^T   H_t(\i)^{\star}  \Big\{ \sum_{\j}  \omega(\j) \mathbf{G}_t (\i-\j)\Big\},
\end{align}
 where the vector-valued image $\mathbf{G}_t  : \Omega \rightarrow \mathbb{C}^d$ is given by $\mathbf{G}_t (\i) = H_t(\i) \f(\i)$. Similarly, on substituting \eqref{finalApproxY} in \eqref{redifnormal}, we obtain
\begin{equation}
\label{denom}
Z(\i)=\frac{1}{T} \sum_{t=1}^T   H_t(\i)^{\star}  \Big\{ \sum_{\j}  \omega(\j) H_t (\i-\j)\Big\}.
\end{equation}
Notice that we have managed to express \eqref{redifBF} and \eqref{redifnormal} using a series of Gaussian convolutions. Indeed, the term within brackets in \eqref{denom} is a Gaussian convolution, which we denote by $\omega \ast H_t$. On the other hand, the  Gaussian convolution in \eqref{numerator} is performed on each of the $d$ components of $\mathbf{G}_t$; we denote this using $\omega \ast {\mathbf{G}}_t$. Thus, we are required to perform a total of $T(d+1)$ Gaussian convolutions, which constitutes the bulk of the computation. The procedure for computing the proposed approximation of \eqref{BF} is summarized in Algorithm \ref{algo}. We shall henceforth refer to the proposed algorithm as the \texttt{MCSF} (Monte Carlo Shiftable Filter). 
Note that we can efficiently compute the Gaussian convolutions in step \ref{conv} using separability and recursion, and importantly, at $O(1)$ cost with respect to $\sigma_s$ \cite{Deriche1993}. The run-time of the proposed algorithm is thus independent of $\sigma_s$, and is hence expected to be much faster than the direct implementation for large $\sigma_s$.
Needless to mention, we can replace the spatial filter $\omega(\i)$ in \eqref{kernel} with any arbitrary spatial filter (such as a box or hat filter), and the above reductions would still hold. 

\begin{algorithm}
\KwData{Image $\f : \Omega \rightarrow \mathbb{R}^d$, and parameters $\sigma_s,\mathbf{C},N$, and $T$.}
\KwResult{Shiftable approximation of \eqref{BF} denoted by $\mathcal{S}[\f]$.}
Diagonalize $\mathbf{C}$ to get $\mathbf{Q}$ and $\alpha_1,\ldots,\alpha_d$\;
Set up the spatial filter $\omega(\i)$ in \eqref{kernel} using $\sigma_s$\;
\For{$k =1,2,3$}{
$\gamma_k=\alpha_k/\sqrt{N}$\;
}
\For{$\i \in \Omega$}{
Set $\g(\i)= \mathbf{Q}^T \!\f(\i)$\;
Set $\mathbf{P}(\i) = \mathbf{0}$\;
Set $Z(\i) = 0$\;
}
\For{$t=1,\ldots,T$}{
Draw $\X$ from $B(N,1/2)^d$\;
Set $H(\i) = 1$ for $\i \in \Omega$\;
\For{$k=1,\ldots,d$}{
$ H(\i) = H(\i)\exp\!\big ( \iota (N-2X_k) \gamma_k  g_k(\i)\big)$\;
}
Set $\mathbf{G}(\i) = H(\i) \f(\i)$ for $\i \in \Omega$\;
Compute $\bar{\mathbf{G}} = \omega \ast \mathbf{G}$ and $\bar{H}=\omega \ast H$\; \label{conv}
Set $\mathbf{P}(\i) = \mathbf{P}(\i) +H(\i)^{\star} \bar{\mathbf{G}}(\i)  $ for $\i \in \Omega$\;
Set $Z(\i) = Z(\i) + H(\i)^{\star} \bar{H}(\i)  $ for $\i \in \Omega$\;
}
Set $\mathcal{S}[\f](\i)= Z[\i]^{-1} \mathbf{P}(\i)$ for $\i \in \Omega$. \label{divide}
\caption{Monte Carlo Shiftable Filter (\texttt{MCSF}).}
\label{algo}
\end{algorithm}

\begin{figure}
\centering
\subfloat[\textit{Peppers} \cite{BM3Dimages}.]{\includegraphics[width=0.32\linewidth]{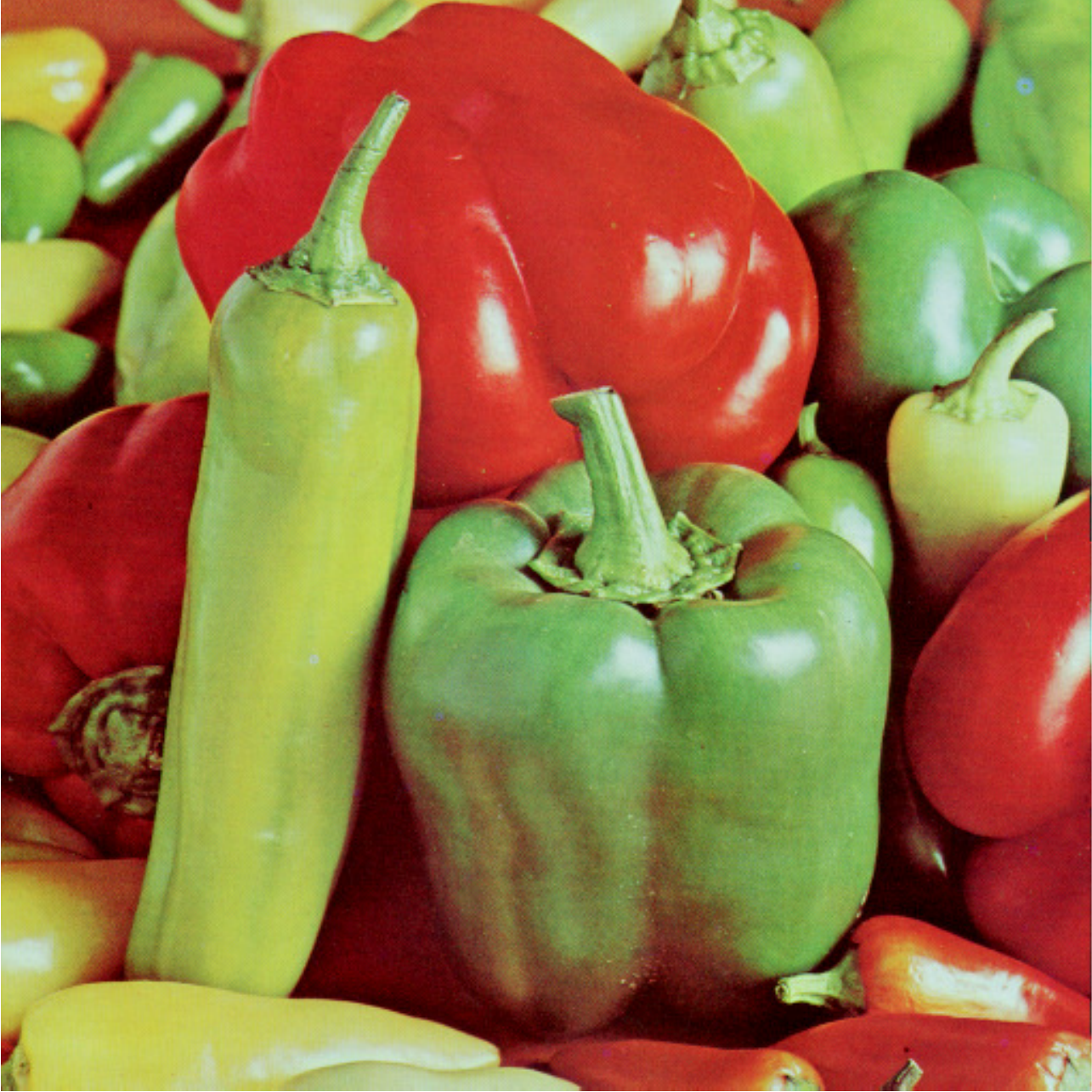}} \hspace{0.2mm} 
\subfloat[Direct, $38$ min.]{\includegraphics[width=0.32\linewidth]{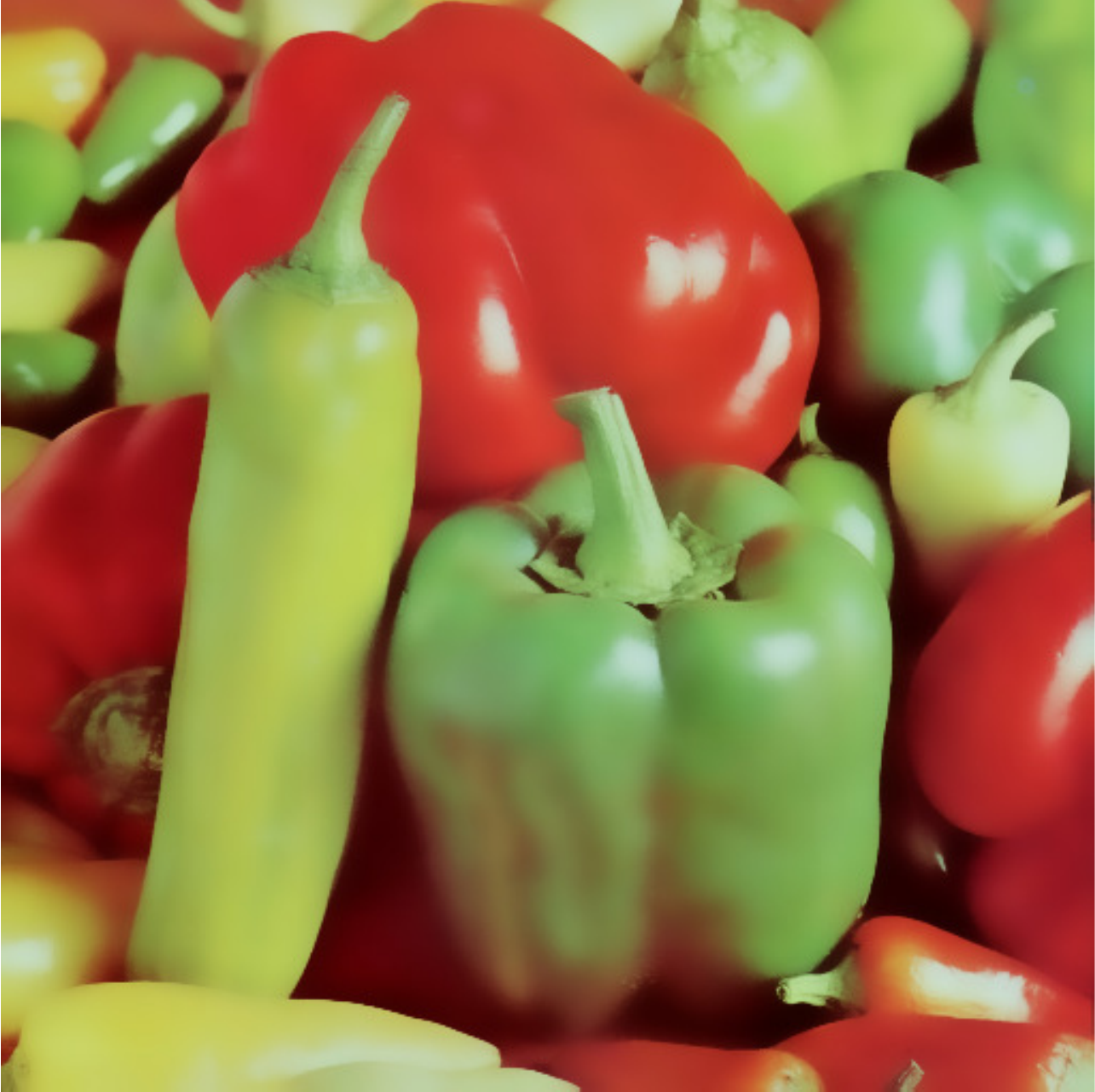}} \hspace{0.2mm}
\subfloat[\texttt{MCSF}, $16$ sec.]{\includegraphics[width=0.32\linewidth]{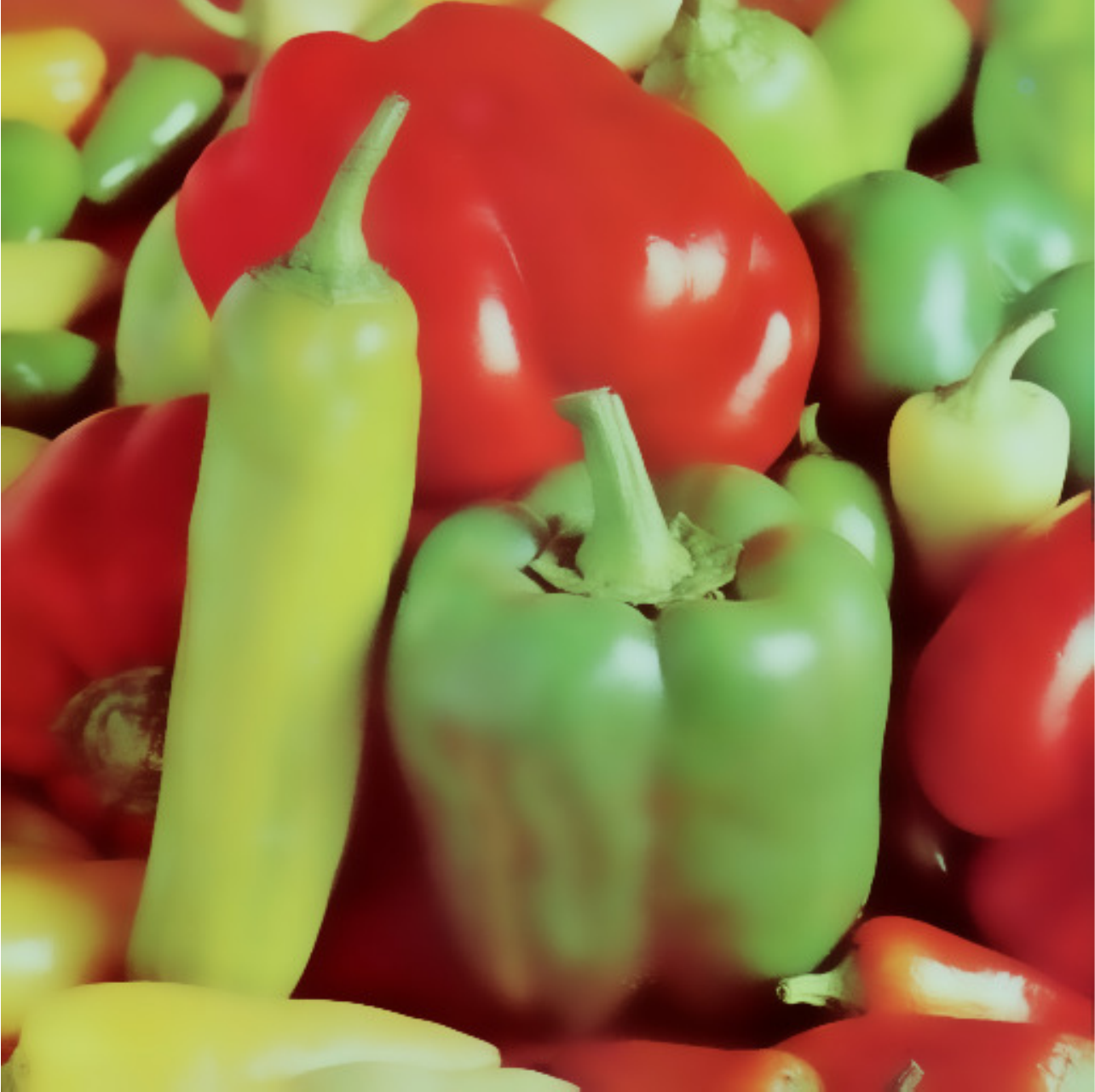}} 
\caption{Results for the $512 \times 512$ \emph{Peppers} image at $\sigma_s=5$ and $\sigma_r = 80$. Parameters for \texttt{MCSF} are  $N = 10$ and $T=300$. The run-times are given in the caption. The MSE  between (b) and (c) is $0.34$ dB.}
\label{Peppers}
\end{figure}

\begin{figure}
\centering
\subfloat[\textit{House} \cite{BM3Dimages}.]{\includegraphics[width=0.322\linewidth]{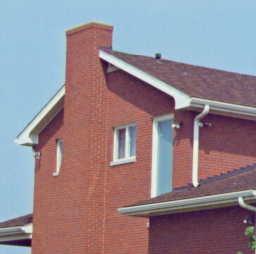}} \hspace{0.5mm} 
\subfloat[Direct, $8.5$ min.]{\includegraphics[width=0.32\linewidth]{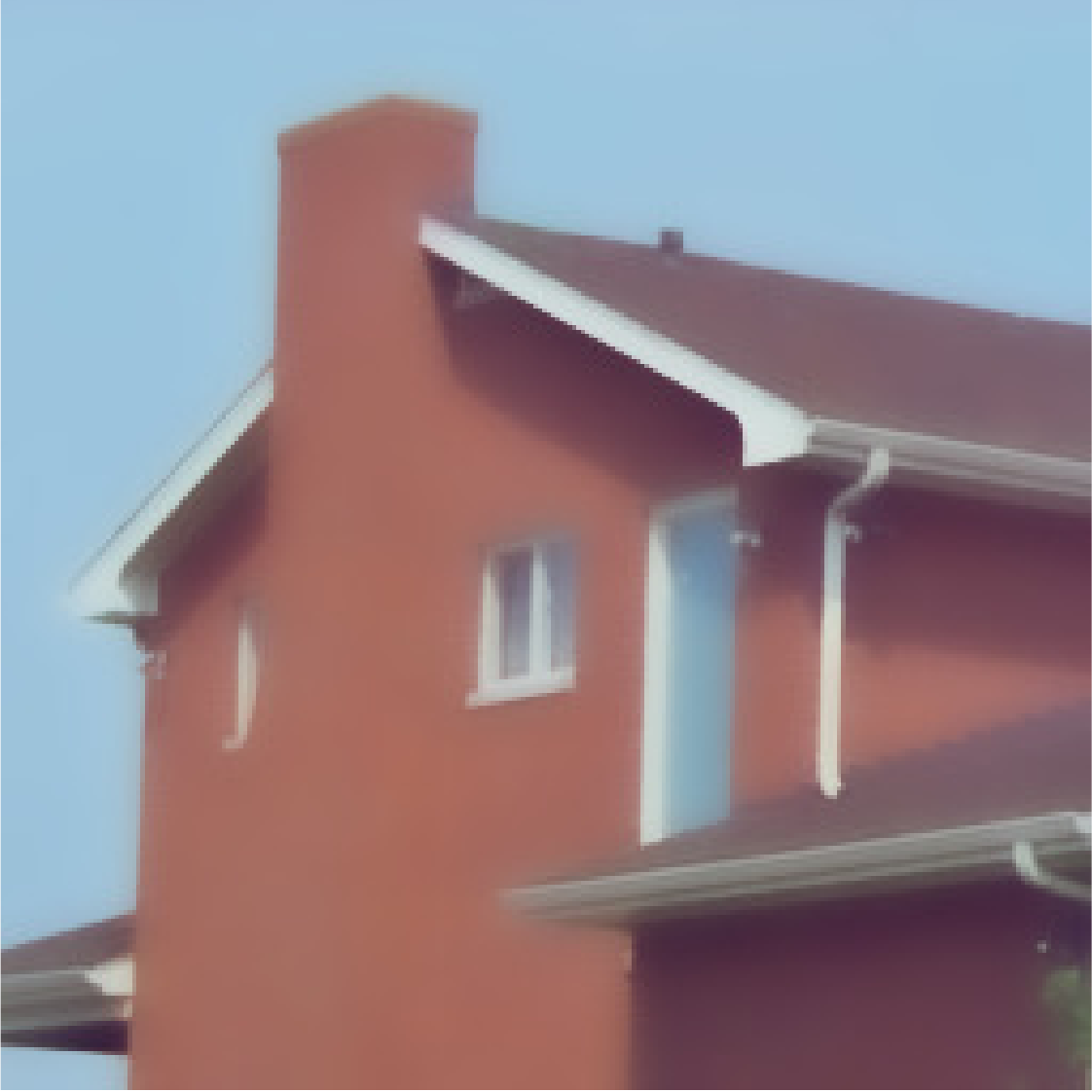}} 
\hspace{0.5mm} 
\subfloat[\texttt{MCSF}, $5$ sec.]{\includegraphics[width=0.32\linewidth]{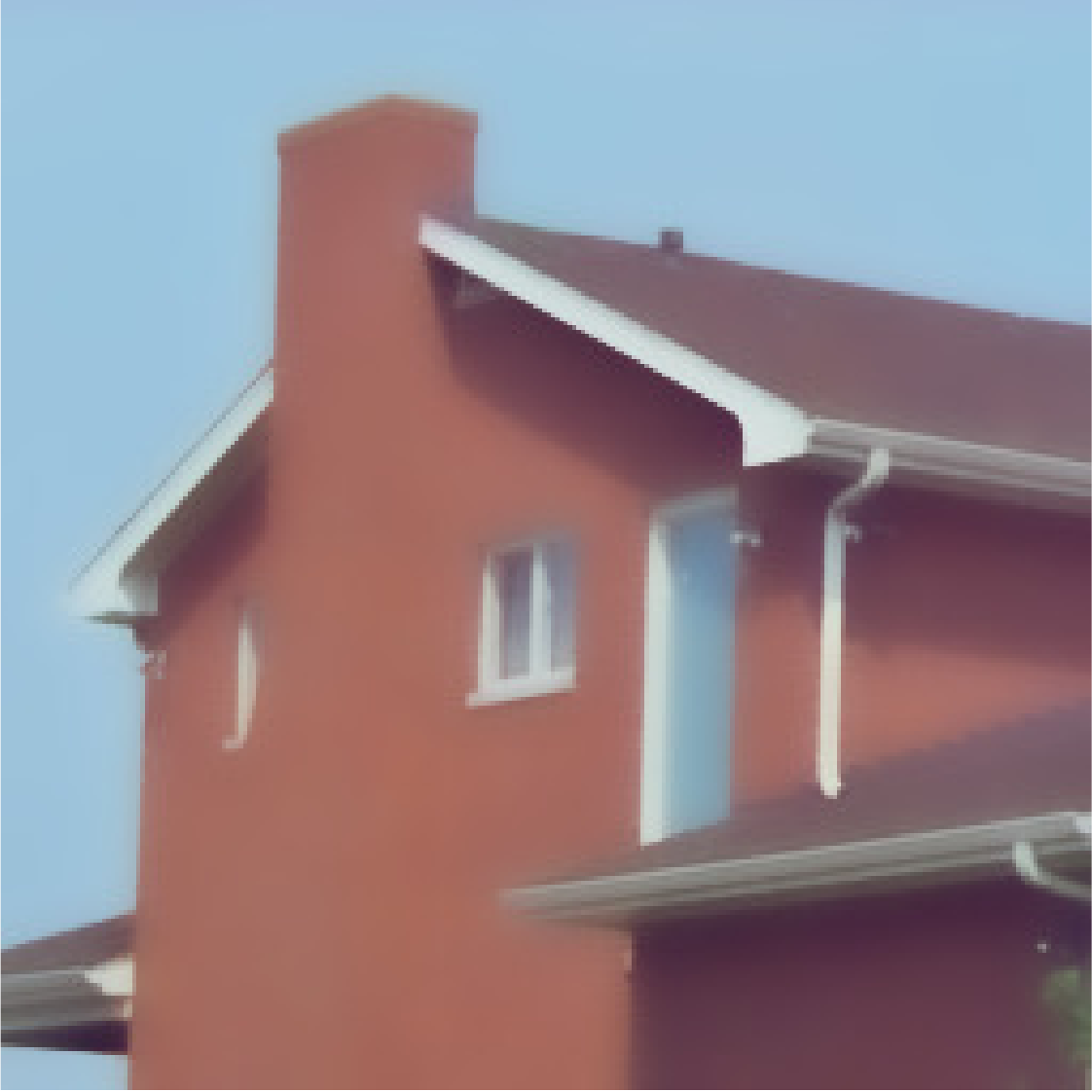}} 
\caption{Results for the $256 \times 256$ \emph{House} image at $\sigma_s=5$ and
$\sigma_r = 90$. The parameters for \texttt{MCSF} are $N = 10$ and $T=300$. The MSE between (b) and (c) is $-1.48$ dB.} 
\label{House}
\end{figure}

\section{Results for Color Images}
\label{results}

In this section, we present some results on natural color images for which $d=3$. 
In particular, we demonstrate that the proposed \texttt{MCSF} algorithm is both fast and accurate for color images in relation to the direct implementation.
To quantify the approximation accuracy for a given color image $\f : \Omega \rightarrow \mathbb{R}^3$, we used the mean-squared error between $\mathcal{B}[\f]$ and $\mathcal{S}[\f]$ (the latter is the output of Algorithm \ref{algo}) given by
\begin{equation}
\label{MSE}
 \text{MSE}=\frac{1}{3 | \Omega | } \sum_{k=1}^{3} \sum_{\i \in \Omega} \big(\mathcal{B}[\f]_k(\i)-\mathcal{S}[\f]_k(\i) \big)^2,
\end{equation}
where $\mathcal{B}[\f]_k$ and $\mathcal{S}[\f]_k$ are the $k$-th color channel. On a logarithmic scale, this corresponds to $10 \log_{10}(\text{MSE})$ dB. For the experiments reported in this paper, we used an isotropic Gaussian kernel corresponding to $\mathbf{C}= \sigma_r^2 \mathbf{I}$ in \eqref{kernel}. In other words, $\mathbf{Q}=\mathbf{I}$, and $\alpha_k=1/\sigma_r$ for $k=1,2,3$.

The accuracy of the proposed algorithm is controlled by the order of the raised-cosine ($N$) and the number of trails ($T$). 
It is clear that we can improve the approximation accuracy by increasing $N$ and $T$.
We illustrate this point with an example in Figure \ref{errorPlot}. We notice that \texttt{MCSF} can achieve sub-pixel accuracy when $N = 10$ and $T > 300$. 
We have noticed in our simulations that, for a fixed $T$, the accuracy tends to saturate beyond a certain $N$. This is demonstrated in Figure \ref{errorPlot} using $N=10$ and $N=20$.

\begin{figure}
\centering
\includegraphics[width=0.83 \linewidth]{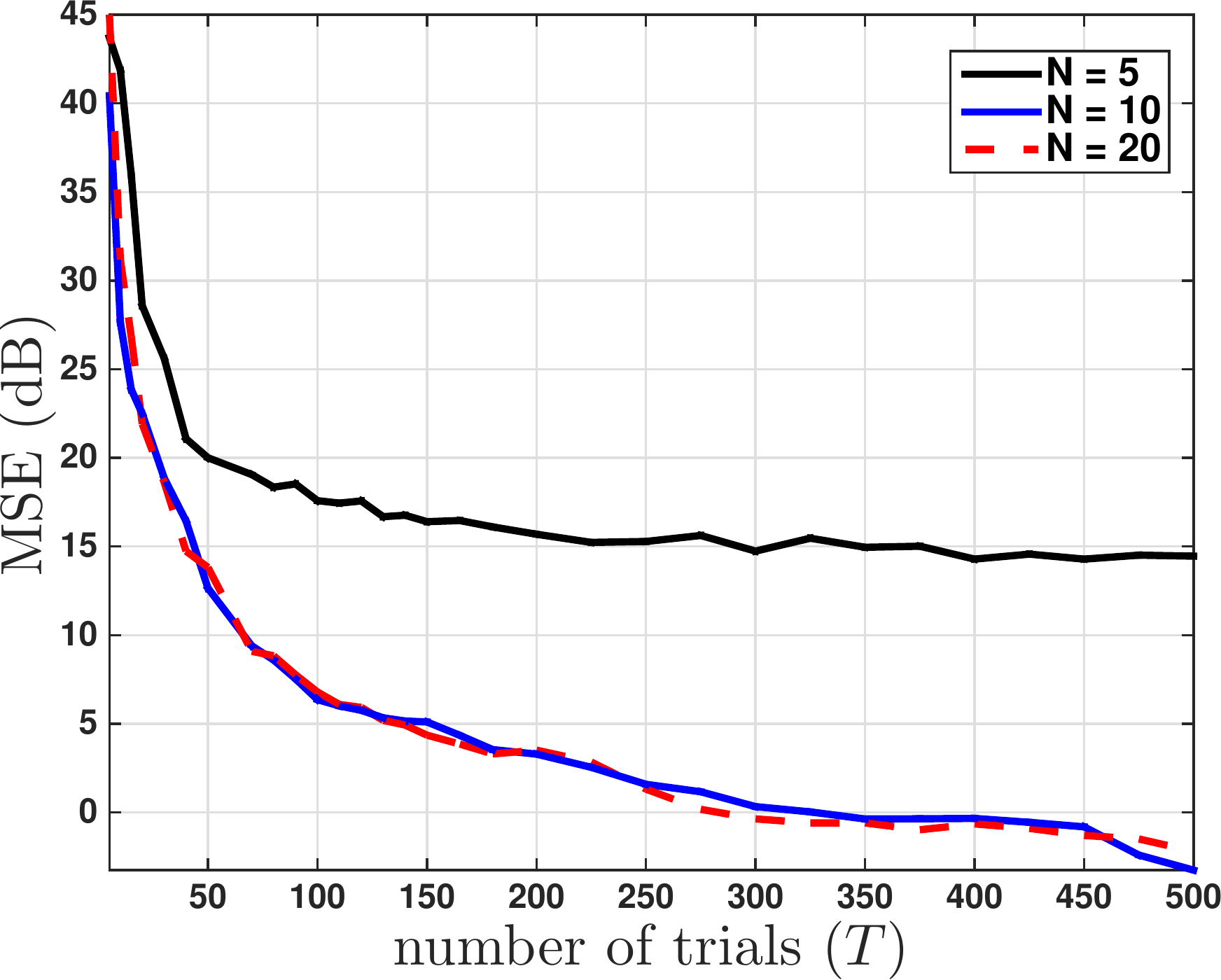} 
\caption{Plot of the MSE given by \eqref{MSE} as a function of the number of Monte Carlo trails. The \emph{Dome} image \cite{Paris2006} was used for this experiment, and the parameters are $\sigma_s = 1$ and $\sigma_r = 30$. For a fixed $N$ and $T$, the MSE was averaged over $400$ Monte Carlo realizations.}
\label{errorPlot}
\end{figure}

\begin{table}
\setlength{\tabcolsep}{4.2pt}
\caption{Comparison of the run-time of the direct implementation and Algorithm \ref{algo} on the $512 \times 512$ \emph{Peppers} image. The range parameter used  is $\sigma_r = 40$. The parameters of \texttt{MCSF} are $N = 10$ and $T = 300$. The computations were performed using Matlab $8.4$ on a $3.40$ GHz Intel $4$-core machine with $32$ GB memory.}  
\vspace{2mm}
\centering 
\begin{tabular}{|c| c c  c c c c |} 
\hline
Method $\backslash$ $\sigma_s$     & $1$    & $2$ & $3$  & $4$ & $5$ & $10$  \\
\hline
Direct        & 116s     & 6.6m   & 14.6m  & 24.1m    & 36.5m  & 141m \\
\hline
\texttt{MCSF}     & 16.9s  & 17.1s  & 17.2s  & 17.3s & 17.3s &  17.5s   \\
\hline
\end{tabular}
\label{table1}
\end{table}

A comparison of the run-time of the direct implementation of \eqref{BF} and that of the proposed algorithm is provided in Table \ref{table1}. We notice that \texttt{MCSF} is few orders faster than the direct implementation, particularly for large $\sigma_s$. Indeed, following the fact that the convolutions in step \ref{conv} of Algorithm \ref{algo} can be computed in constant-time with respect to $\sigma_s$ \cite{Deriche1993}, our algorithm has $O(1)$ complexity with respect to $\sigma_s$. 
As against this, the direct implementation scales as $O(\sigma_s^2)$.

Finally, we present a visual comparison of the filtering for RGB images in Figures \ref{Dome} and \ref{Peppers}. Notice that the outputs are visually indistinguishable. As mentioned earlier, 
the authors in \cite{Tomasi1998, Paris2006} have observed that the application of the bilateral filter in the RGB color space can lead to color leakage, particularly at the sharp edges. The suggested solution was to perform the filtering in the CIE-Lab space \cite{WS1982}. In this regard, a comparison of the filtering in the CIE-Lab space is provided in Figure \ref{House}. In this case, we first performed a color transformation from the RGB to the CIE-Lab space, performed the filtering in the CIE-Lab space, and then transformed back to the RGB space. The filtered outputs are seen to be close, both visually and in terms of the MSE.

\section{Conclusion}
\label{conclusion}

We proposed a fast algorithm for the bilateral filtering of vector-valued images. We applied the algorithm for filtering color images in the RGB and the CIE-Lab space. In particular, we demonstrated that a speedup of few orders can be achieved using the fast algorithm without introducing visible changes in the filter output (the latter fact was also quantified using the MSE). An important theoretical question arising from the work is the dependence of the order and the number of trials on the filtering accuracy. In future work, we will investigate this matter, and also look at various ways of improving the Monte Carlo integration \cite{DW2005}. We also plan to test the algorithm on other vector-valued images.

\bibliographystyle{IEEEbib}

\end{document}